\documentclass[letterpaper]{article} 
\usepackage{aaai23}  
\usepackage{times}  
\usepackage{helvet}  
\usepackage{courier}  
\usepackage[hyphens]{url}  
\usepackage{graphicx} 
\usepackage{natbib}  
\usepackage{caption} 
\usepackage{algorithm}
\usepackage{algorithmic}

\usepackage{newfloat}
\usepackage{listings}
\DeclareCaptionStyle{ruled}{labelfont=normalfont,labelsep=colon,strut=off} 
\floatstyle{ruled}
\newfloat{listing}{tb}{lst}{}
\floatname{listing}{Listing}
\title{Promises and Challenges of Causality for Ethical Machine Learning}
\author{
    Aida Rahmattalabi, 
    Alice Xiang
}
\affiliations{
    SonyAI \\

    aida.rahmattalabi@sony.com,
    alice.xiang@sony.com
}

\usepackage{bm}
\usepackage{amsmath,amsfonts,amsthm,mathrsfs,mathtools}
\usepackage[graphicx]{realboxes}
\usepackage{booktabs}
\usepackage{multirow}
\usepackage{bm}
\usepackage{color}
\usepackage{subcaption}
\usepackage{caption}
\usepackage{centernot}
\usepackage{tabularx}
\usepackage{array}
\newcolumntype{P}[1]{>{\centering\arraybackslash}p{#1}}

\newtheorem{assumption}{Assumption}
\newtheorem{definition}{Definition}

\newtheorem{theorem}{Theorem}

\newtheorem{lemma}{Lemma}
\newtheorem{observation}{Observation}

\newtheorem{proposition}{Proposition}

\newcommand{\E}{\mathbb E} 
\renewcommand{\P}{\mathbb P}

\begin{document}

\maketitle

\begin{abstract}
Recent years have witnessed an increasing interest in causal reasoning for designing fair decision-making systems due to its compatibility with legal frameworks, interpretability for human stakeholders, and robustness to spurious correlations, among other factors. The recent attention, however, has been accompanied with great skepticism due to practical and epistemological challenges with applying existing causal fairness approaches. Motivated by the long-standing empirical work on causality in econometrics, social sciences, and biomedical sciences, in this paper we lay out the conditions for appropriate application of causal fairness under the ``potential outcomes framework.'' To this end, we propose a shift from postulating interventions on \emph{immutable social categories} to their \emph{perceptions} and highlight two key aspects of interventions that are largely overlooked in the causal fairness literature: timing and nature of manipulations. We argue that such conceptualization is key in evaluating the validity of causal assumptions and conducting sound causal analysis including avoiding post-treatment bias. Further, we illustrate how causality can address the limitations of existing fairness metrics that depend upon statistical correlations. Specifically, we introduce causal variants of common statistical fairness notions, and we make a novel observation that under the causal framework there is no fundamental disagreement between different fairness criteria. Finally, extensive experiments on synthetic and real-world datasets including a case study on police stop and search decisions demonstrate the efficacy of our framework in evaluating and mitigating unfairness.
\end{abstract}

\maketitle
\section{Introduction}
Recently, there has been a growing interest in applying causality for unfairness evaluation and mitigation~\cite{Kusner2017CounterfactualFairness, Kilbertus2017AvoidingReasoning, nilforoshan2022causal}. Causality provides a conceptual and technical framework for addressing questions about the effect of (hypothetical) interventions on, in this context, sensitive attributes such as race, gender, etc. This is in contrast with fairness criteria that merely rely on passive observations~\cite{kamiran2013Quantifying,Hardt2016EqualityLearning, Zafar2017FairnessMistreatment,Kleinberg2017InherentScores}. These statistical criteria achieve fairness by constraining the relationships between observed variables, often in conflicting ways. In fact, it has been shown that it is impossible to satisfy these criteria simultaneously on a dataset~\cite{Kleinberg2017InherentScores, Chouldechova2017FairInstruments,SolonBarocasMoritzHardt2020FairnessOpportunities}. Causality, on the other hand, shifts the focus from association to the identification of \emph{sources of disparities} which is also more compatible with legal requirements of evaluating algorithmic bias~\cite{Xiang2021ReconcilingBias}. 

Nevertheless, causal fairness has too been subject to criticism. One objection is around the validity of the assumptions in causal modeling. The majority of recent research on causal fairness has focused on structural causal models, which encode the relationships between variables via a Directed Acyclic Graph (DAG)~\cite{Kusner2017CounterfactualFairness,Nabi2018FairOutcomes,Chiappa2019PathspecificFairness}. Existing approaches assume the DAG is known or employ causal discovery methods~\cite{Malinsky2018CausalGuide}. Causal discovery from observational data, however, relies on strong untestable assumptions which are often unexamined. 
This is problematic, as slight errors in the models can have significant impact on fairness conclusions~\cite{binkyte2022causal}.

There are also concerns about considering social categories as a cause~\cite{Kohler-Hausmann2019EddieDiscrimination,Hu2020WhatsLearning,Kasirzadeh2021TheLearning}. From one perspective, these attributes are largely determined at the time of conception. Hence, they are modeled as source nodes in a DAG which can directly or indirectly influence the descendent variables. Under this view, one needs to identify all problematic pathways through which a sensitive attribute influences the outcome to evaluate and mitigate unfairness. In addition to the modeling challenge of this view, in practice, a single entity may not be held liable for the discrimination along an entire causal pathway. Indeed, many anti-discrimination mechanisms investigate whether \emph{an individual or institutional actor} has behaved in a discriminatory manner. For example, in an employment setting, a racial discrimination lawsuit aims to determine whether a firm has withheld benefits, e.g., hiring, with regard to the race of the applicant. In reality, racial disparities in hiring rates might be a reflection of either discrimination or differences in the applicant's qualifications. That is, if past discrimination in the educational system has led to some applicants having lower educational achievements, by hiring based on educational achievements, the employer will perpetuate the effects of this discrimination. Under anti-discrimination law, however, as long as the employer makes the hiring decision based on educational achievements that are legitimately connected with the business needs---with no regard to race---no liability is attached. In fact, if the employer seeks to proactively address past societal discrimination, this could lead to reverse discrimination lawsuits~\cite{Mitchell2013AnCase}. Another concern is post-treatment bias that arises when we condition on intermediate outcomes that themselves may be the product of discrimination, resulting in biased estimates of the treatment effect~\cite{Rosenbaum1983TheEffects}. Since sensitive attributes such as gender are largely determined at the time of one's conception, arguably almost all measurable variables become post-treatment. Thus, conditioning on those variables may lead to misleading estimates of discrimination. Removing those variables, as proposed in~\cite{Kusner2017CounterfactualFairness,Nabi2018FairOutcomes}, also leaves little to no information for valid causal analysis. 
Finally, there are concerns about the epistemological and ontological aspects of interventions on social categories~\cite{Hu2020WhatsLearning,Kasirzadeh2021TheLearning}. Specifically,~\citet{Hu2020WhatsLearning} criticize DAG-based models that present the sensitive category as an isolated node, whereas in reality, these social categories are constituted by multi-dimensional complex features. 


In this work, we argue that not every dataset lends itself to a valid causal analysis. Precisely, we build on the work of~\citet{Greiner2011CausalCharacteristics} and lay out the conditions for the appropriate application of causal fairness under the ``potential outcomes framework.'' To this end, we highlight two important aspects that are widely overlooked in the causal fairness literature: the timing and nature of manipulations. In particular, we distinguish between interventions on immutable attributes and their perceptions. The idea that perceptions matter and can be manipulated is not new. For example, researchers have examined the effect of manipulated names associated with political speeches~\cite{Sapiro1981IfImages} and resumes~\cite{Bertrand2004AreDiscrimination}. Nevertheless, fair machine learning (ML) literature has paid little attention to its consequences for valid unfairness evaluation. We discuss the impact of this specification on the plausibility of causal assumptions and show how this view can help alleviate conceptual difficulties proposed in previous work. Additionally, we make the following contributions:


\begin{itemize}
\item We propose a causal framework to investigate and mitigate unfairness of a particular actor's decisions and provide a sequential imputation approach that adjusts the post-treatment variables before being used for inference.
\item We illustrate how causality can address the limitations of existing fairness criteria, including those that depend upon statistical correlations. In particular, we introduce the causal variants of the popular statistical criteria of fairness and make a novel observation that under the causal framework, there is indeed no fundamental disagreement between different fairness definitions. 
\item We conduct extensive synthetic and real-world experiments where we demonstrate the effectiveness of our methodology for unfairness evaluation and mitigation compared to common baselines. Our results indicate that the causal framework can effectively identify and remove disparities at various stages of decision-making.
\end{itemize}
\section{Related Work}\label{sec:related}
There are two main frameworks for causal inference: structural causal models (SCM)~\cite{Hitchcock2001Causality:Inference}, and the potential outcomes framework (POF)~\cite{Rubin2005CausalOutcomes}. According to the SCM framework, the relationship between variables is modeled explicitly via a DAG. Causal queries are performed by intervening on the treatment variable and propagating its effect through the DAG. Under POF, each individual is characterized by a vector of potential outcomes which is only partially observed. Generally speaking, DAGs encode more assumptions about the relationships of the variables~\cite{SolonBarocasMoritzHardt2020FairnessOpportunities}. Consequently, POF has been more widely adopted in empirical research, including unfairness evaluation outside of ML~\cite{Bertrand2004AreDiscrimination,VanderWeele2014OnVariables}. 
Specifically, \citet{gaebler2022causal} consider the timing of perceiving one's race in a prosecutorial setting and issues of post-treatment bias. Their analysis, however, is limited to estimation and not fair prediction. In fair ML literature, the majority of works have adopted the DAG framework~\cite{Kusner2017CounterfactualFairness,Kilbertus2017AvoidingReasoning,Zhang2018FairnessFormula,Zhang2018EqualityApproach,Madras2019FairnessData,Chiappa2019PathspecificFairness} with a few exceptions that rely on POF~\cite{Nabi2018FairOutcomes,Khademi2019FairnessCausality}. \citet{Nabi2018FairOutcomes} and~\citet{Chiappa2019PathspecificFairness} study path-specific discrimination, where the former proposes to remove the descendants of the sensitive attribute under the unfair pathway and the latter aims to correct those variables. In many practical settings, it is critical to identify and correct for discrimination imposed by specific actors --- not an entire causal pathway. Our work makes this distinction by specifying the precise timing of hypothetical intervention. Further,~\citet{Khademi2019FairnessCausality} propose causal definitions of group fairness and show how these quantities can be estimated from data. However, their approach is limited to fairness evaluation. Also, it does not distinguish between pre- and post-treatment variables essential in many studies of discrimination~\cite{gaebler2022causal}. A thorough review of related work is provided in the appendix. 
\section{Fairness: Potential Outcomes Perspective}\label{sec:potential}
We consider a decision-making scenario where we aim to make a (discrete) decision $Y \in \{0, 1\}$. Each individual has a set of basic attributes $\bm X \in \mathbb{R}^{n}$ and belongs to the social category $A \in \{0, 1\}$ for which fair treatment is important. We assume $A$ is a single binary variable, however, our discussion can naturally be extended to cases where $A$ has more than two levels e.g., by comparing every group against a baseline group. It also applies when there is more than one sensitive attribute, such as the intersection of race and gender, by considering their joint values. Causality views the unfairness evaluation as a counterfactual inference problem. For example, we aim to answer questions of the type: \emph{What would have been the hiring decision, if the person had been perceived to be of a different gender?} 
Such causal questions are centered around the notion of an \emph{intervention} or \emph{treatment} on a social category. Building on the POF framework of~\citet{Rubin2005CausalOutcomes}, we define $\bm Y = (Y(0), Y(1))$ as random variables describing the potential outcomes under different interventions on $A$. For any individual, only the value of $Y(A)$ corresponding to the perceived $A$ is observable. 

We take a decision-maker's perspective, considering how their perception of one's sensitive attribute may lead to different decisions. Through this conceptualization, it is possible for discrimination to operate at various points in time and across multiple domains throughout one's life. For example, in the context of racial discrimination, earlier work has recognized potential points of discrimination, including the labor market, education, etc.~\cite{Quillian2006MeasuringDiscrimination}. 
Consequently, we need to specify the point in time at which we wish to measure and mitigate unfairness. In causal terms, this is closely related to the notion of the timing of the intervention, i.e., the time at which one's sensitive attribute is perceived by an actor. To illustrate, consider a hiring scenario and suppose we are interested in evaluating whether the hiring decision is fair with respect to gender or not. We can investigate unfairness at different stages, e.g., from the first time an individual comes into contact with the company (e.g., resume review), progresses in the system (during interviews), or when the final decision is being made. We may even take a much broader perspective and investigate the effect of gender from the point an individual attends college and study how gender affects education and subsequently the opportunities in the job market. Indeed, as we expand our view, the causal inference problem becomes more intractable, but the conceptual framework remains valid.

\begin{figure}[t!]
    \centering
    \includegraphics[width = 0.45\textwidth]{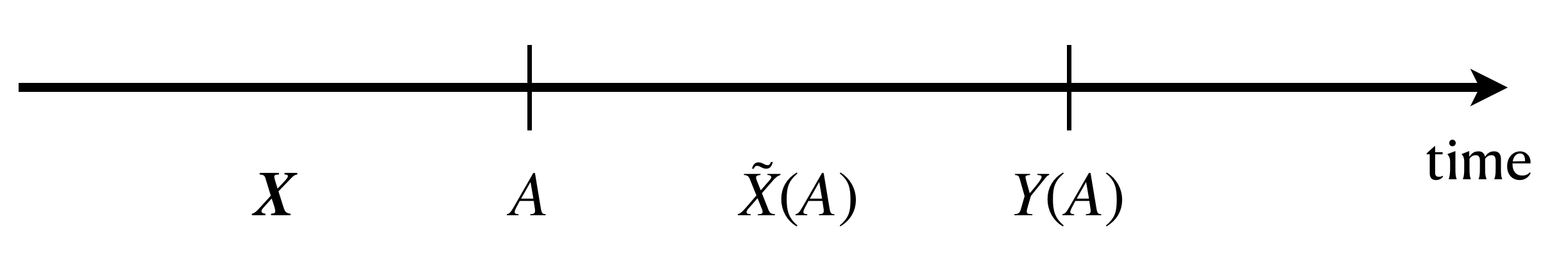}
    \caption{The time when sensitive attribute $A$ is perceived determines pre- and post-treatment variables. Here, $\bm X$ and $\tilde{X}(A)$ are pre- and post-treatment variables, respectively.}
    \label{fig:timeline}
\end{figure}


Specifying the time of the intervention is also important for conducting sound causal analysis. Consider an interview process that is discriminatory, resulting in unfair interview scores for a group. In our fairness analysis, if we condition on  the interview score, we will find no relationship between gender and hiring decision. This is due to post-treatment bias cautioned in the causal inference literature that happens when variables fixed after an intervention are used in evaluating the treatment effect~\cite{Montgomery2018HowIt}. Once we fix the time of intervention, variables $\tilde{X} \in \mathbb R$ determined afterward are considered as post-treatment variables and in principle are affected by $A$ (Figure~\ref{fig:timeline}). Thus, we can postulate a vector $\tilde{\bm {X}} = (\tilde{X}(0), \tilde{X}(1))$. 

The nature of the intervention also influences the causal effect that we uncover. For instance, in the study conducted in~\cite{Bertrand2004AreDiscrimination}, the authors manipulated the names on the resumes to measure racial discrimination which only allowed them to capture the level of discrimination exhibited through the relationship between one's name and perception of race. Under a different manipulation, e.g., zip code of the applicant, the outcome of the study would have been different. In observational studies, where the analyst has no control over how an individual's sensitive attribute is perceived, a careful examination of perception mechanisms is still necessary. In particular, when there are several mechanisms affecting perceived attributes (e.g., name, zip code, etc.), it may be possible to study their joint effect by modeling the missing counterfactual values, under each mechanism.

Building on the above discussion, we define fairness as the absence of a causal effect due to a \emph{specific intervention} on a perceived sensitive attribute at a \emph{particular time}. 


\begin{definition}[Causal Parity]
A decision-making process achieves causal parity if $\E[Y(1) - Y(0)] = 0$. 
\end{definition}\label{def:CP}
Here, $\E[Y(1) - Y(0)]$ is the treatment effect of $A$ on $Y$. 
Since both potential outcomes $Y(0), Y(1)$ are not simultaneously observed for an individual, several assumptions are necessary to identify the treatment effect from the data. 


\begin{assumption}[Stable Unit Treatment Value Assumption (SUTVA)]
The treatment that one unit (individual) receives does not change the potential outcomes of other units.
\end{assumption}

\begin{assumption}[Consistency]
The outcome $Y$ should agree with the potential outcome under the respective treatment, i.e., $Y=Y(0)(1-A)+Y(1)A$. In other words, there are no two ``flavors'' or versions of treatment such that $A=1$ under both versions but the
potential outcome for $Y$ would be different under the alternative versions.
\end{assumption}

\begin{assumption}[Positivity]
At each level of pre-treatment variables $\bm X$, the probability of receiving any form of treatment is strictly positive. Mathematically,
$$\mathbb{P}\left( \mathbb{P}(A = a \mid \bm X = \bm x) > 0 \right) = 1 \; \forall a \in \{0, 1\}, \bm x \in \mathcal X.$$
\end{assumption}

\begin{assumption}[Conditional Exchangeability]
Conditioned on pre-treatment variables $\bm X$, individuals receiving the treatment should be considered exchangeable with those not receiving
the treatment. Mathematically, 
$$
{\bm Y}, \tilde{X} \perp A \mid \bm X = \bm x \; \forall \bm x \in \mathcal X.
$$
\end{assumption}

Earlier works have emphasized the criticality of these assumptions in determining the causal effects~\cite{Rubin2007BayesianRandomization}. We highlight their importance for fairness evaluation. SUTVA can be viewed as a non-interference assumption and depends very much on the problem and the choice of the decision-maker. For example, for a recruiter as the decider, one should think carefully whether the recruiter's decision to proceed with an application is independent from case to case. If a recruiter screened three exceptional female candidates in a row, maybe they would treat the fourth candidate's gender differently. On the one hand, they might think, ``wow, the women are so impressive,'' and that positive feeling would extend to regarding the fourth female candidate positively. On the other hand, they might think, ``I can't only hire women,'' and thus give a boost to subsequent male candidates. In this case, SUTVA is violated as historical data on other candidates influences the future candidates outcomes.
The consistency assumption means that there are no hidden levels of treatment variation. For example, if we only have binary gender data available, the assumption would be that the employer does not treat non-binary individuals in a different way from their labeled gender. This assumption emphasizes the importance of thinking carefully about the taxonomy of sensitive attribute labels and how they are collected. In order to accurately measure discrimination against non-binary individuals, we would need to have more nuanced gender perception labels and to conceptualize the treatment as having multiple levels. Consistency can also be viewed as treatment invariance, which we discussed in the previous section in the context of nature of intervention on social categories. When intervening on social categories such as race, it is possible that different factors contribute to the perception of one's sensitive attribute. Under consistency, one needs to make sure that there is sufficient data in order to capture the different levels of ``race.'' Without such nuanced data, it is still possible to measure the causal effect, but the interpretation changes, as the estimated causal effect is an average of multiple potential treatments. 
The positivity assumption is also essential in order identify the treatment effect. It requires that there is not a complete overlap between the treatment assignment and pre-treatment variables. For example, if all of the women in a hiring pool have a PhD, and all of the men only have a Master's degree, then it is not possible to separate the effects of gender discrimination from the educational attainment. 

Conditional exchangeability implies that there are no hidden confounder and is in principle impossible to verify in observational studies. In practice, analysts include as many pre-treatment variables as possible to ensure that as many confounders as possible are accounted for. The goal is to ensure that once all of the  pre-treatment variables $\bm X$ are controlled for, the allocation of individuals between treatment and control is as close to random as possible. In the fairness setting, this would mean, after controlling for $\bm X$, the only systematic difference between the two groups is the perception of their sensitive attribute, allowing for an empirical estimate of the effect of discrimination. We note that we have the conditional independence of the counterfactuals of both $\tilde{X}$ and $\bm Y$, a key distinction between our work and earlier work that does not differentiate between pre- and post-treatment variables~\cite{Khademi2019FairnessCausality}. In multi-stage decision-making settings, we have more than one choice of decision-maker to study. In such cases, an analyst may have to balance the need to make the exchangeability assumption plausible against the desire to study a decision-maker's behavior early in the decision-making chain. Choosing the timing of the intervention towards the later interactions renders more measured variables pre-treatment which in turn can make the exchangeability assumption more plausible. However, by treating such variables as pre-treatment the analyst forgoes the detection of any prior discrimination. 

\section{Trade-offs under the Lens of Causality}\label{sec:tradeoff}
We now turn to an important aspect of our study which is to study common statistical criteria under the causal lens.  

\subsection*{Causal Fairness Definitions}
We center our discussion on the criteria with known impossibility results in the fair ML literature.
\begin{definition}[Conditional Causal Parity]
A decision-making process achieves conditional causal parity if $$\E[Y(1) - Y(0) \,\vert\, \bm X = \bm x] = 0 \, \forall \bm x \in \mathcal X. $$
\end{definition}\label{def:CCP}
The above definition is closely related to conditional statistical parity which aims to evaluate fairness after controlling for a
limited set of ``legitimate'' factors~\cite{kamiran2013Quantifying}. The set of legitimate factors is typically assumed as given, e.g., by domain experts and 
significantly impacts the conclusions we draw. In contrast, in our definition $\bm X$ collects all the pre-treatment variables. Hence, once the nature of the intervention is explicitly defined, all remaining pre-treatment variables can be considered as legitimate since the main effect we aim to identify is the effect of the treatment. 
In this definition, we can substitute $Y$ with predicted outcome $\hat{Y}$. Subsequently, a classifier satisfies conditional causal parity if $\E[\hat{Y}(1) - \hat{Y}(0) \mid \bm X = \bm x] \, \forall \bm x \in \mathcal X.$ {Here, $\hat{Y}(A)$ is not an estimate of $Y(A)$, rather, it represents the the prediction in a world in which individuals are treated as either $A = 0$ or $A = 1$. Contrary to statistical criteria that evaluate fairness by comparing the average predictions between two groups in the observed data, the causal criteria compares two (counterfactual) worlds. We will elaborate further on this when we present the theoretical results. 
}


\begin{definition}[Causal Equalized Odd]
A predictor $\hat{Y}$ satisfies causal equalized odds if:
\begin{align*}
    \P(\hat{Y}(0) = 1 \mid Y(0) = 1) = \P(\hat{Y}(1) = 1 \mid Y(1) = 1) \\
    \P(\hat{Y}(0) = 1 \mid Y(0) = 0) = \P(\hat{Y}(1) = 1 \mid Y(1) = 0)
\end{align*}
\end{definition}

The above definition is the causal counterpart of equalized odds proposed by~\citet{Hardt2016EqualityLearning}. 
It states that the probability of receiving a true positive and false positive prediction in worlds where everyone is treated as $A=0$ or $A=1$ should be the same. Therefore, an individual does not have any preferences to be in either of these worlds. 
\begin{definition}[Causal Calibration]
A risk score $S \in \mathcal{S}$ is well-calibrated within groups if:
\begin{align*}
    & \P(Y(0) = 1 \mid S(0) = s) = \P(Y(1) = 1 \mid S(1) = s)\; \forall s \in \mathcal S. 
    \end{align*}
\end{definition}
In words, for a risk score to be calibrated, the proportion of positive outcomes in either worlds should be equal. 
\begin{definition}[Causal Positive Predictive Parity] 
A predictor $\hat{Y}$ satisfies causal positive predictive parity if:
\begin{align*}
    & \P(Y(0) = 1 \mid \hat{Y}(0) = 1) = \P(Y(1) = 1 \mid \hat{Y}(1) = 1).
    \end{align*} 
\end{definition}
Causal predictive parity has a similar interpretation as causal calibration. It requires the rate of positive outcomes in either worlds to be the same when the prediction is positive. 

\subsection*{Trade-offs among Causal Criteria of Fairness}
We now investigate two main impossibility results among the statistical fairness criteria and demonstrate that there is no fundamental disagreement between their causal variants. Proofs of this section are deferred to the appendix.

\noindent\textbf{Causal Parity and Conditional Causal Parity.}
As shown by the Berkeley college admission study, statistical parity and conditional statistical parity may not be simultaneously satisfied on a dataset. According to the study, female applicants were admitted at a lower rate compared to the male~\cite{Bickel1975SexBerkeley}. However, after controlling for department choice, the disparity vanished. 
\begin{observation}
There exists a distribution $p(\bm X, A, Y)$ such that conditional statistical parity does not imply statistical parity, i.e., $\E[Y \mid \bm X = x, A = 1] - \E[Y \mid \bm X = x, A = 0] = 0 \; \forall \bm x \in \mathcal X \centernot\implies \E[Y \mid A = 1] - \E[Y \mid A = 0] = 0.$
\label{obs:impossible-parity}
\end{observation}
In contrast, it is straightforward to show that conditional causal parity implies causal parity which suggests an inherent compatibility between the two fairness notions. 
\begin{proposition}
Conditional causal parity $\Rightarrow$ causal parity. 
\label{prop:possible-causal-parity}
\end{proposition}
The proof follows simply from taking the expectation over $\bm X$. The intuition behind the above result is that $\E[Y \mid A]$ merely measures the statistical dependence between $Y$ and $A$ and does not differentiate between different sources of dependence, e.g., female students applying for more competitive departments than male students, or a discriminatory admission process. We note that conditional causal parity is a more stringent requirement than causal parity and the reverse implication does not generally hold true. 


\noindent\textbf{Causal Positive Predictive Parity and Causal Equalized Odds.}
It is well-known that one can not achieve positive predictive parity or calibration together with equalized odds simultaneously unless either the base rates $\P(Y \mid A = a)$ are equal or the classifier is perfect~\cite{Kleinberg2017InherentScores,Chouldechova2017FairInstruments}. 
Here, we show no such restrictions are necessary for their causal variants.  
\begin{theorem}
If $\E[Y(1) - Y(0)] = 0$ (causal parity), there exists a classifier that satisfies causal positive predictive parity and causal equalized odds simultaneously. 
\label{thm:main-theorem}
\end{theorem}
In fact, an example of a classifier that satisfies all causal fairness criteria is one that uses only the pre-treatment variables $\bm X$. That is because, these variables precede the perception (treatment). Thus, they are unaffected by it.

We now conclude that a classifier that satisfies conditional causal parity satisfies the remaining causal criteria. This result indicates that conditional causal parity is the most stringent fairness requirement, sufficient to ensure other criteria are enforced. It is also in sharp contrast with known impossibility results in fair ML literature. We note that causal parity, i.e.,  $P(Y(0) = 1) = P(Y(1) = 1)$, is satisfied if decisions are made regardless of one's group membership and is different from the equal base rate assumption which does not necessarily hold in many applications. Finally, evaluating and enforcing these criteria is challenging partly due the missing counterfactual. Next, we discuss our approach for unfairness evaluation and mitigation using data imputation.

\subsection*{Fairness Evaluation}
Various approaches exist in order to estimate the causal effect~\cite{imbens2015causal}. In this work, we focus on regression modeling as it fits well with the imputation framework. We first consider a case where post-treatment variables are absent. Under causal assumptions, it holds:
\begin{align*}
& \E[Y(1) - Y(0)] = \E\left[\E\left[Y(1) - Y(0) \mid \bm X = \bm x\right]\right] = \\
& \E\left[\E\left[Y \mid \bm X = \bm x, A = 1\right] - \E\left[Y \mid \bm X = \bm x, A = 0\right]\right],
\end{align*}
which can be estimated from observational data regressing outcome $Y$ on $\bm X$ and $A$. When post-treatment variables $\bm X$ are present, simply conditioning on those variables will introduce bias. Instead, we should take their dependence on $A$ into account, i.e., $\bm X(A)$. With a slight abuse of notation, we can write the treatment effect of $A$ as $\E[Y(1, \tilde{X}(1)) - Y(0, \tilde{X}(0))]$, where potential outcomes are indexed by both the treatment and the post-treatment counterfactuals. In the mediation literature, this quantity is known as \emph{total effect}~\cite{Imai2010IdentificationEffects}.
Estimating the total effect poses a considerable identification challenge as it depends on four $\tilde{X}(0), \tilde{X}(1), Y(0, \tilde{X}(0)), Y(1, \tilde{X}(1))$ counterfactuals which are not simultaneously observed for any individual. To tackle this problem, we draw on well-established ideas in the causal inference literature to develop a sequential imputation technique similar to those used in social sciences to evaluate the long-term impact of policy shifts~\cite{Xiang2015AssessingSystem}.
We illustrate using an example with only one post-treatment variable $\tilde{X}$. First, we learn a model by regressing $\tilde{X}$ on $\bm X$ and $A$. We then use this model to impute the counterfactuals $\tilde{X}(A)$ by substituting the corresponding value of $A$. Next, we learn a regression model, by mapping $\bm X$, observed $\tilde{X}$ and $A$ to $Y$. To obtain potential outcomes $Y(A)$, we substitute $A$ and $\tilde{X}(A)$ with their respective values. The process can be generalized to multiple post-treatment variables by repeating the sequential imputation conditioned on previously observed variables. 

\subsection*{Unfairness Mitigation}
{Our approach transforms the inference problem on
$p(\bm X, \tilde{X}, Y)$ into an inference problem on another distribution $p^{\star}(\bm X, \tilde{X}, Y)$, in which the post-treatment variables, including the outcome of interest are \emph{adjusted}. In a sense, $p^{\star}$ represents a ``fair world'' in which individuals are treated regardless of their $A$ value. The idea of adjusting variables affected by sensitive attributes has been recently investigated in the context of mitigating path-specific effects under the DAG framework~\cite{Nabi2018FairOutcomes,Chiappa2019PathspecificFairness}. In our work, we are interested in mitigating unfairness attributed to a specific actor's decision at a particular point in time, rather than a specific entire causal path.




To adjust the downstream variables, we develop a pre-processing approach that pools from counterfactuals $\tilde{X}(A)$ and $\bm Y(A)$ to create new counterfactuals that are equal for both groups. To illustrate, we first consider a setting with no post-treatment variables and assume $\E[Y(1) - Y(0)] \neq 0$. Let $A = 0$ be the baseline group. The baseline group can be viewed as either the majority group or a historically advantaged group. If we had access to $Y(0)$ for every individual in the population, we could correct the $Y(1)$ by setting it to $Y(0)$, enforcing $\E[Y(1) - Y(0)] = 0$. Subsequently, we could learn a fair classifier as e.g., $\arg\max_y \P(Y(0) = y \mid \bm X = \bm x)$. This way we are effectively eliminating decision-maker's unfavorable attitude towards membership in group $A=1$. In the presence of post-treatment variables, we employ a similar approach, after adjusting the post-treatment variable in a similar fashion. Precisely, the classifier models $\arg\max_y \P(Y(0) = y \mid \bm X = \bm x, \tilde{X}(0) = \tilde{x}), $. A key challenge with this approach is that $Y(0)$ values are not observed for every individual. Similar to the unfairness evaluation section, we leverage imputation from causal inference literature to tackle this problem~\cite{Rubin1996MultipleYears}. We note that the pooling can be conducted in various ways, including using the counterfactuals corresponding to $A = 1$, $\max(Y(0), Y(1))$, or other complex functions of the outcomes under each counterfactual world. We leave the empirical investigation of either of these methods to future work.  

\section{Empirical Results}\label{sec:experiments}
We compare our approach against three commonly used unfairness mitigation methods from the literature: a pre-processing method ReW~\cite{Kamiran2012DataDiscrimination}, 
an in-processing method PRem~\cite{kamishima2012fairness}, 
and a post-processing method ROC~\cite{kamiran2012decision}.
We use the implementations in \texttt{AIF360} library~\cite{bellamy2019ai}. We study the effect of timing of intervention on fairness conclusions and present evaluation results with respect to both statistical and causal criteria presented in Section~\ref{sec:tradeoff}. Existing causal approaches require a DAG model as input and typically remove path-specific effects. Thus, they do not apply in our setting. 

\paragraph{Synthetic example.}
We begin with a stylized hiring scenario. We consider a decision-making process that involves two stages: interview and final hiring decision. 
We use $A$ to represent gender, which we draw from a Bernoulli distribution $\textit{Bern}(0.75)$ with the majority class being male $A=1$. An individual's qualification is described by a random variable $X$ drawn from a normal distribution $\mathcal{N}(2\alpha(A - 0.5), 1)$, where $\alpha$ controls the difference in the average qualifications between genders. Each candidate has a score $S$ which we model as a binary variable such that $\mathbb P(S = 1) = \sigma\left(2X + 2\beta (A - 0.5)\right), $ where $\sigma(z) = 1/(1+e^{-z})$ is the logistic function and $\beta \geq 0$ determines the level of discrimination in $S$, e.g., when $\beta > 0$ being a male $A = 1$ increases one's probability of receiving a higher score. Subsequently, a hiring decision $Y$ is made according to the probabilistic model $\mathbb{P}(Y = 1) = \sigma\left(2X + S +  2\gamma(A - 0.5)\right)$, with $\gamma \geq 0$ controlling the level of discrimination in $Y$ for a fixed $X, S$. The counterfactual values, for both $S(A)$ and $Y(A)$, can be obtained by substituting the corresponding value of $A$ in the above models. We note that for analysis before the interview stage, we use the counterfactual values of $S(A)$ in order to obtain the potential outcomes of $Y(A)$. 
We generate 100,000 samples according to the described model and use logistic regression 
to impute the missing counterfactuals. We study the effect of timing of the intervention, pre and post-interview, on our conclusions across different values of $\alpha, \beta$ and $\gamma$. The full set of results is presented in the appendix.

\begin{figure*}[t]
    \centering
    \includegraphics[width=0.98\textwidth]{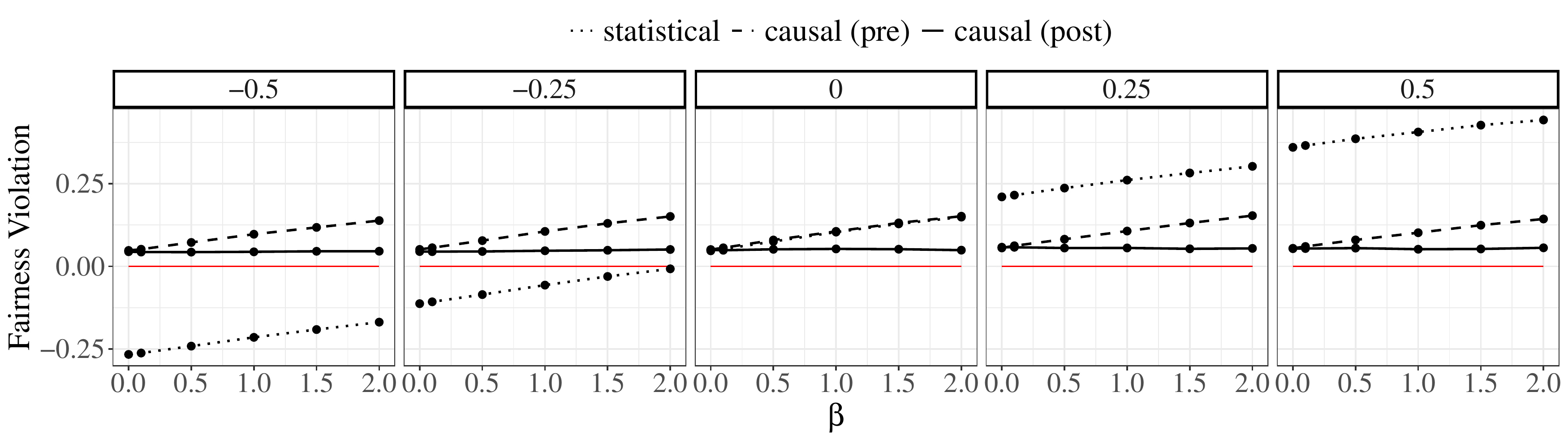}
    \caption{Stylized hiring scenario: line types indicate the unfairness evaluation approach. Causal (pre) and causal (post) measure pre- and post-interview causal disparity, respectively, and are compared against the statistical disparity line. Each panel corresponds to a different $\alpha$. The $y-$axis exhibits violation of parity fairness criteria across values of $\beta$ and $\gamma = 0.2$.}
    \label{fig:synthetic}
\end{figure*}



Figure~\ref{fig:synthetic} depicts a summary of unfairness evaluation findings across different values of $\alpha$ and $\beta$. First, we observe the post-interview causal plot remains almost flat across different values of $\beta$ exhibiting a constant causal disparity (due to $\gamma \neq 0$), independent of prior discrimination in the interview stage. This suggests that early discrimination will remain undetected if one chooses a later time as the point of intervention. Pre- and post-interview lines only intersect at $\beta = 0$ and pre-interview fairness violation increases monotonically with $\beta$. Statistical line exhibits  significantly different results. For example, when $\alpha = -0.5$, The statistical line lies below the causal ones which suggests that it underestimates the true level of discrimination. The reason is that when $\alpha < 0$, males' qualification is lower than females on average. However, since $\beta, \gamma > 0$ the interview score and the final decision are in favor of male candidates. Since statistical parity fails to disentangle these opposing effects, it results in lower estimates of unfairness. On the other hand, when $\alpha > 0$, these effects reinforce each other leading to an over-estimation of unfairness. Only when $\alpha = 0$, do statistical parity and causal parity (causal-(pre) line) match which indicates the sensitivity of statistical criteria to baseline differences between groups (average qualifications).

We now analyze different mitigation approaches. We set $\alpha = 0, \beta = 0.25$ and $\gamma = 0.2$ to simulate a scenario where all the disparities are due to the discrimination. We then compare the ability of different approaches in mitigating such unfair disparities. Due to the inherent compatibility of the causal criteria, we mainly focus on causal parity and defer the full set of results to the appendix. The casual parity violation of the baseline models are: ReW (0.074), PRem (0.086) and ROC (0.140) compared to the 0.105 causal disparity in the data. These numbers represent the outcome disparities an individual receives if they were of different gender, with the male being the favorable gender. These disparities are because once the gender changes, so does any downstream variables that these models depend on. Causal (pre) and causal (post) on the other hand, adjust the downstream observations by creating a ``fair world'' resulting in zero disparities between gender by design. We also provide a summary of the results for different statistical metrics in Table~\ref{tab:mitigation}. Notably, causal (pre) outperforms baselines w.r.t to the parity criteria (0.009). Causal (post) exhibits higher disparity which is due to the fact that it only mitigates any unfairness post interview. Statistical methods exhibit varying behavior, showing less effectiveness in reducing outcome disparity compared to other metrics. This is in part due to the fact that these models use features that are themselves affected by the sensitive attribute. Compared to other metrics, our model does not necessarily outperform all other models, however, it should be noted that these metrics rely on the observed unfair labels. Thus, they are not always a reliable measure for unfairness evaluation.

\paragraph{Police Stop and Search. }
We also test our approach on standardized police stop data made available by Stanford Open Policing Project~\cite{pierson2020large}. 
Each row represents a stop. Features include time, reason or location of the stop, driver's race and gender, as well as intermediate outcomes such as whether or not the search was conducted, contraband was found or an arrest was made. The multiple decision points present in this dataset make it possible to study the impact of the timing of the intervention on the causal analysis. We focus on the racial disparities. Details on the data preparation as well as a discussion on the plausibility of the causal assumptions are presented in the appendix. 


\begin{table*}[t!]
    \small
    \centering
    \begin{tabular}{P{2.5cm}P{1.5cm}P{1.5cm}P{3cm}P{2cm}P{2cm}P{1.5cm}}
    \toprule
    \multirow{1}{*}{} & \multirow{1}{*}{Domain} & \multicolumn{5}{c}{Fairness Violation (Statistical)}  \\
    \cmidrule{3-6} &  & Parity & Positive Pred. Parity & E-Odds (TP) & E-Odds (FP) & Accuracy\\
    \hline
    ReW  & \multirow{5}{*}{Hiring} & 0.092 &  0.054 &  \textbf{0.028} &  \textbf{0.016} & \textbf{0.808} \\
    PRem &  & 0.141 & \textbf{0.008} & 0.100 &  0.046 & 0.772  \\
    ROC & & 0.028 & 0.085 & -0.036 & -0.045 & 0.798 \\
    Causal (post) &  & 0.085 &  0.023 & \textbf{0.028} &  {0.018} &  0.768  \\
    Causal (pre) & & \textbf{0.009} &  0.109 & -0.039 & -0.069 & 0.788 \\
    \hline
    ReW  & \multirow{5}{*}{Policing} & -0.094 &  0.098 & 0.022 & -0.040 & \textbf{0.933}\\
    PRem & & -0.076 &  0.085 & 0.055 & -0.031 & 0.924 \\
    ROC & & -0.096 &  0.104 & 0.018 & -0.041 & 0.932 \\
    Causal (arrest) & & -0.065 & 0.089 & -0.043 & -0.016 & 0.927 \\
    Causal (search) &  & \textbf{-0.011} & \textbf{-0.064} & \textbf{0.014} & \textbf{-0.011} & 0.921 \\
    \bottomrule
    \end{tabular}
\caption{Fairness violation of statistical criteria and classification accuracy. Bold numbers indicate the best value in each column.}
\label{tab:mitigation}
\end{table*}


We use three neural networks with two hidden layers of size 8 to impute the post-treatment variables: whether or not an individual is searched, the search outcome and the arrest outcome. In the search outcome stage, we filter the counterfactual outcomes based on the counterfactual search decisions. That is if the search is not conducted had the individual were of a particular race, the corresponding search outcome is set to ``nothing found.'' Based on the imputed counterfactuals, we evaluate causal parity in subsequent decisions. In what follows, we use $S$, $C$ and $Y$ to denote whether an individual was searched, contraband was found or an arrest was made. For intervention at \emph{Search Stage}:
 $\E[S(\texttt{White}) - S(\texttt{Black}) ] = -5.5\%$. Subsequently, $\E[C(\texttt{White}) - C(\texttt{Black}) ] = -1.5\%$ which is relatively small given that black individuals are searched more often. Finally, $\E[Y(\texttt{White}) - Y(\texttt{Black}) ] = -8.2\%$. Statistical parity violation is also -8.3\%. We compare this against intervention at \emph{Arrest Stage}, where $\E[Y(\texttt{White}) - Y(\texttt{Black}) ] = -3.9\%$. This indicates that by moving the time point of the causal analysis to a future time, we are uncovering smaller discriminatory effects. Once again, this underscores the importance of specifying the timing of the intervention, which is absent in purely statistical analyses. Next, we train a causal model using the outcome under the condition of being perceived as White. The causal parity violation of baseline models is: ReW (-0.0258), PRem (-0.024) and ROC (-0.028), compared to zero disparity of our causal approach. Table~\ref{tab:mitigation} also summarizes the results of statistical criteria. As seen, our causal model outperforms the existing models along all criteria without much loss in accuracy. We also observe that we can significantly improve the parity measure by accounting for the discriminatory effects in the search stage compared to the arrest stage.


\section{Conclusion}


In this work, we illustrated the utility of applying concepts from the ``potential outcomes framework'' to algorithmic fairness
problems. 
We argue that fairness evaluation is not a static problem and unfairness can happen at various points and within and across multiple domains. This is in contrast with methods that rely on fixed DAG models. Next, we demonstrated how such a causal framework can address the limitations of existing approaches. Specifically,
our theoretical investigation indicates that there is an inherent compatibility between the causal fairness definitions we
propose. Finally, we showed the effectiveness of our approach in evaluating and mitigating unfairness associated with
different stages of decision-making. We hope that our empirical observations spark additional work on collecting new
datasets that lend themselves to temporal fairness evaluation.

\bibliography{aaai23}

\onecolumn
\appendix
\section{Supplemental Material}

\subsection{Related Work}
There are two main frameworks for causal inference: structural causal models~\cite{Hitchcock2001Causality:Inference} that represent the relationship between variables via a DAG, and the potential outcomes framework (POF)~\cite{Rubin2005CausalOutcomes}. DAGs can be viewed as a sequence of steps for generating a distribution from independent noise variables. Causal queries are performed by changing the value of a treatment variable and propagating its effect through the DAG~\cite{Hitchcock2001Causality:Inference}. POF, on the other hand, postulates potential outcomes under different interventions, albeit some unobserved. In general, DAGs encode more assumptions about the relationships of the variables; i.e., one can derive potential outcomes from a DAG, but potential outcomes alone are not sufficient to construct the DAG. Consequently, POF has been more widely adopted in empirical research, including bias evaluation outside of ML~\cite{Bertrand2004AreDiscrimination,VanderWeele2014OnVariables, gaebler2022causal}. More detailed discussion on the differences between the two frameworks in relation to empirical research can be found in~\cite{Imbens2020PotentialEconomics}. Causal inference on immutable attributes has appeared in several works including~\cite{VanderWeele2014OnVariables,Kilbertus2017AvoidingReasoning, gaebler2022causal}. Specifically,~\citet{gaebler2022causal} consider manipulations on perceptions of sensitive attributes. The authors study the timing of perceiving one’s race in a prosecutorial setting and issues of post-treatment bias. Their analysis, however, is limited to estimation and not fair
prediction.


In machine learning, there has been a rising interest in causal reasoning for designing fair prediction/decision-making systems, where the majority of works have adopted the DAG framework~\cite{Kusner2017CounterfactualFairness,Kilbertus2017AvoidingReasoning,Zhang2018FairnessFormula,Zhang2018EqualityApproach,Madras2019FairnessData,Chiappa2019PathspecificFairness} with a few exceptions that rely on POF~\cite{Nabi2018FairOutcomes,Khademi2019FairnessCausality}. Specifically,~\citet{Kusner2017CounterfactualFairness} provide an individual-based causal fairness definition that renders a decision fair towards an individual if it is the same in the actual world and a counterfactual world where the individual possessed a different sensitive attribute. Their model requires a full specification of the structural causal model, including the structural equations and the DAG.~\citet{Kilbertus2017AvoidingReasoning} propose \emph{proxy discrimination} as (indirect) discrimination via proxy variables such as name, visual features, and language which are more amenable to manipulation. The removal
of proxy discrimination, however, depends on the functional form of the causal dependencies which is not available in practice.~\citet{Nabi2018FairOutcomes} and~\citet{Chiappa2019PathspecificFairness} study path-specific discrimination, where the former proposes to remove the descendants of the sensitive attribute under the unfair pathway and the latter aims to correct the those variables. In practice, we are often interested in discrimination attributed to particular actors. In addition, identifying the problematic causal paths is quite challenging as context-specific. In a similar fashion,~\citet{salimi2019interventional} define causal fairness based on the notion of admissible variables, variables through which it
is permissible for the sensitive attribute to influence the outcome. They propose a data pre-processing approach to enforce it. In the POF realm,~\citet{Khademi2019FairnessCausality} propose two causal definitions of group fairness: fair on average causal effect (FACE), and fair on average causal effect on the treated (FACT) and show how these quantities can be estimated for sensitive attributes as the treatment. The authors restrict their attention to the fairness evaluation task (not prediction) and do not discuss the distinction between pre- and post-treatment variables. Finally,~\citet{Zhang2018FairnessFormula,Zhang2018EqualityApproach} discuss counterfactual direct, indirect, and spurious effects and provides formulas to identify these quantities from observational data. These works too rely on a causal model, or DAG, and develop different methodologies to identify and mitigate unfairness.

Most of the existing works lack a clear discussion of the causal assumptions which hinders the adoption of these methods in practice. In addition, the validity of the causal assumptions are influenced by the nature of the postulated interventions and its timing, which is not clearly articulated in the current literature. In many applications, discrimination by specific individuals or institutional actors is the subject of a study---not an entire causal pathway. Our work makes these distinctions by highlighting the importance of specifying the timing of hypothetical intervention.

We also briefly review the observational/statistical notions of fairness. Demographic parity and its variants have been studied in numerous papers~\cite{Dwork2012FairnessAwareness,Feldman2015CertifyingImpact,Corbett-Davies2017AlgorithmicFairness}. Also referred to as statistical parity, this fairness criteria requires the average outcome to be the same across different sensitive groups. Conditional statistical parity~\cite{Feldman2015CertifyingImpact,Corbett-Davies2017AlgorithmicFairness} imposes a similar requirement after conditioning on a set of legitimate factors. In the classification setting, equalized odds and a relaxed variant, equality of opportunity, have been proposed to measure the disparities in the error rate across different sensitive groups~\cite{Hardt2016EqualityLearning}. These criteria are often easy to state and interpret. However, they do not often appropriately account for relevant differences across individuals, e.g., coding experience in a software development job~\cite{dutta2021fairness}. In addition, they suffer from a major limitation: it is impossible to simultaneously achieve these criteria on any particular dataset~\cite{Kleinberg2017InherentScores, Chouldechova2017FairInstruments,SolonBarocasMoritzHardt2020FairnessOpportunities}. For more detailed discussions, we refer the readers to the survey on fair machine learning~\cite{Mehrabi2021ALearning}. 
In this work, we revisit these notions and introduce their causal variants, where we show that under the causal framework, there is no fundamental disagreement between different criteria.

Finally, different unfairness mitigation approaches have been proposed in the statistical fairness literature.~\citet{Kamiran2012DataDiscrimination} developed a pre-processing technique that weights the examples in each (sensitive group, label) combination differently to ensure fairness before classification.~\citet{kamishima2012fairness} proposed an in-processing algorithm that adds a discrimination-aware regularization term to the learning objective. Post-processing methods aim to modify the classifier's output to satisfy the fairness criteria.~\citet{kamiran2012decision} designed an approach that gives favorable outcomes to unprivileged groups and unfavorable outcomes to privileged groups in a confidence band around the decision boundary. Our approach closely relates to the pre-processing scheme. Precisely, we propose to transform the data distribution to one that belongs to a ``fair world,'' one that treats each individual with no regard to their sensitive attribute. Consequently, we substitute the post-treatment variables, including the outcome of interest, with a value pooled from the counterfactual outcomes under $A = 0$ or $A = 1$. One benefit of the pre-/post-processing approaches is that they are not to a specific classifier, giving more flexibility to the analyst to choose the best model.

\subsection{Proofs of Statements in Section~\ref{sec:tradeoff}}
In this section, we provide formal proofs for the theoretical results in Section~\ref{sec:tradeoff}. 

\begin{proof}[Proof of Proposition~\ref{prop:possible-causal-parity}]
    $\E[Y(1) - Y(0)] = \E [ \E[Y(1) - Y(0) \mid \bm X = \bm x ] ] = \E[0] = 0.$
\end{proof}

\begin{proof}[Proof of Theorem~\ref{thm:main-theorem}]

\begin{align}
    \P(\hat{Y}(1) = 1 \mid Y(1) = 0) & = \P(\hat{Y}(1) = 1 \mid Y(1) = 0)\frac{\P(Y(1) = 0)}{1 - \P(Y(1) = 1)} \nonumber\\ 
    & = \frac{\P(\hat{Y}(1) = 1) - \P(\hat{Y}(1) = 1 \mid Y(1) = 1)\P(Y(1) = 1)}{1 - \P(Y(1) = 1)} \quad \text{[Total Probability Rule]} \nonumber\\
    & = \frac{\P(\hat{Y}(1) = 1) - \P(Y(1) = 1 \mid \hat{Y}(1) = 1)\P(\hat{Y}(1) = 1)}{1 - \P(Y(1) = 1)}\quad \text{[Bayes Theorem]} \nonumber\\ 
    & = \frac{1 - \P({Y}(1) = 1 \mid \hat{Y}(1) = 1)}{1 - \P(Y(1) = 1)} \P(\hat{Y}(1) = 1) \nonumber\\
    & = \frac{1 - \P({Y}(1) = 1 \mid \hat{Y}(1) = 1)}{1 - \P(Y(1) = 1)} \frac{\P(Y(1) = 1 \mid \hat{Y}(1) = 1)}{\P(Y(1) = 1 \mid \hat{Y}(1) = 1)}\P(\hat{Y}(1) = 1) \nonumber\\
    & = \frac{1 - \P({Y}(1) = 1 \mid \hat{Y}(1) = 1)}{1 - \P(Y(1) = 1)} \frac{\P(\hat{Y}(1) = 1 \mid {Y}(1) = 1)}{\P(Y(1) = 1 \mid \hat{Y}(1) = 1)}\P({Y}(1) = 1) \nonumber\\
    & = \frac{1 - \P({Y}(1) = 1 \mid \hat{Y}(1) = 1)}{\P(Y(1) = 1 \mid \hat{Y}(1) = 1)}\frac{\P({Y}(1) = 1)}{\P({Y}(1) = 0)}{\P(\hat{Y}(1) = 1 \mid {Y}(1) = 1)}.
    \label{eq:1}
\end{align}
In a similar fashion, we can derive the following equality:
\begin{align}
    \P(\hat{Y}(0) = 1 \mid Y(0) = 0) & = \frac{1 - \P({Y}(0) = 1 \mid \hat{Y}(0) = 1)}{\P(Y(0) = 1 \mid \hat{Y}(0) = 1)}\frac{\P({Y}(0) = 1)}{\P({Y}(0) = 0)}{\P(\hat{Y}(0) = 1 \mid {Y}(0) = 1)}.
    \label{eq:2}
\end{align}
As given in the premise of the theorem, $\P(Y(1) =  1) = \P(Y(0) = 1)$ and $\P(Y(1) = 0) = \P(Y(0) = 0)$. Subsequently, 
$$\frac{\P({Y}(1) = 1)}{\P({Y}(1) = 0)} = \frac{\P({Y}(0) = 1)}{\P({Y}(0) = 0)}.$$
From Equations~\eqref{eq:1} and~\eqref{eq:2} and the above equality, it follows that a classifier that satisfies any two fairness criteria will automatically satisfy the last one. 
\end{proof}

\subsection{Experimental Results in Section~\ref{sec:experiments}}

\noindent\textbf{Synthetic. }Figure~\ref{fig:evaluation-synthetic-big} provides more extensive unfairness evaluation results, comparing our causal approach and a purely statistical approach against different their respective fairness criteria. As seen, all statistical criteria fail to capture the true level of discrimination: for negative alpha values, the statistical criteria tend to underestimate unfairness. On the contrary, when $\alpha > 0 $, statistical criteria overestimate unfairness along all four metrics. The causal and statistical lines only overlap at $\alpha = 0$ when there are no baseline differences between groups. Figure~\ref{fig:evaluation-synthetic-big-gamma} also compares approaches across different values of $\gamma$ (discrimination at the hiring stage), where can we see that the causal (post) fairness violation now increases with $\gamma$, albeit with an offset compared to causal (pre). That is because causal (pre) also measures the unfairness in the interview stage that is perpetuated in the hiring stage. 

We now provide the full causal evaluations on different unfairness mitigation techniques in Table~\ref{tab:mitigation-causal}. The models are evaluated as if they are deployed in worlds where one is a male or female. Positive numbers indicate that the decision process is in favor of males. By surgically manipulating one's gender, post-treatment variables are affected. However, what variables are considered post-treatment depends on the timing of the manipulation. Here, we consider the interview stage as our primary point of (hypothetical) manipulation used for all evaluations. We compare causal (pre) and (post) to distinguish between two models that are \emph{trained} based on manipulations at the interview and hiring stage. According to this table, we observe that methods that aim to reduce statistical unfairness may lead to significant disparities under the causal lens, all in favor of males. Causal (post) also has a higher unfairness than causal (pre) which is due to the fact it does not adjust for the unfairness in the interview stage. Nonetheless, it performs competitively compared to the statistical methods. Causal (pre) achieves zero disparity (by design) through adjustment of the post-treatment variables and outcomes. We note that if we evaluate the causal (post) approach based on an intervention in the hiring stage, we would obtain zero disparities for the same reasons as explained above.

\begin{figure}
    \centering
    \includegraphics[width = 0.9\textwidth]{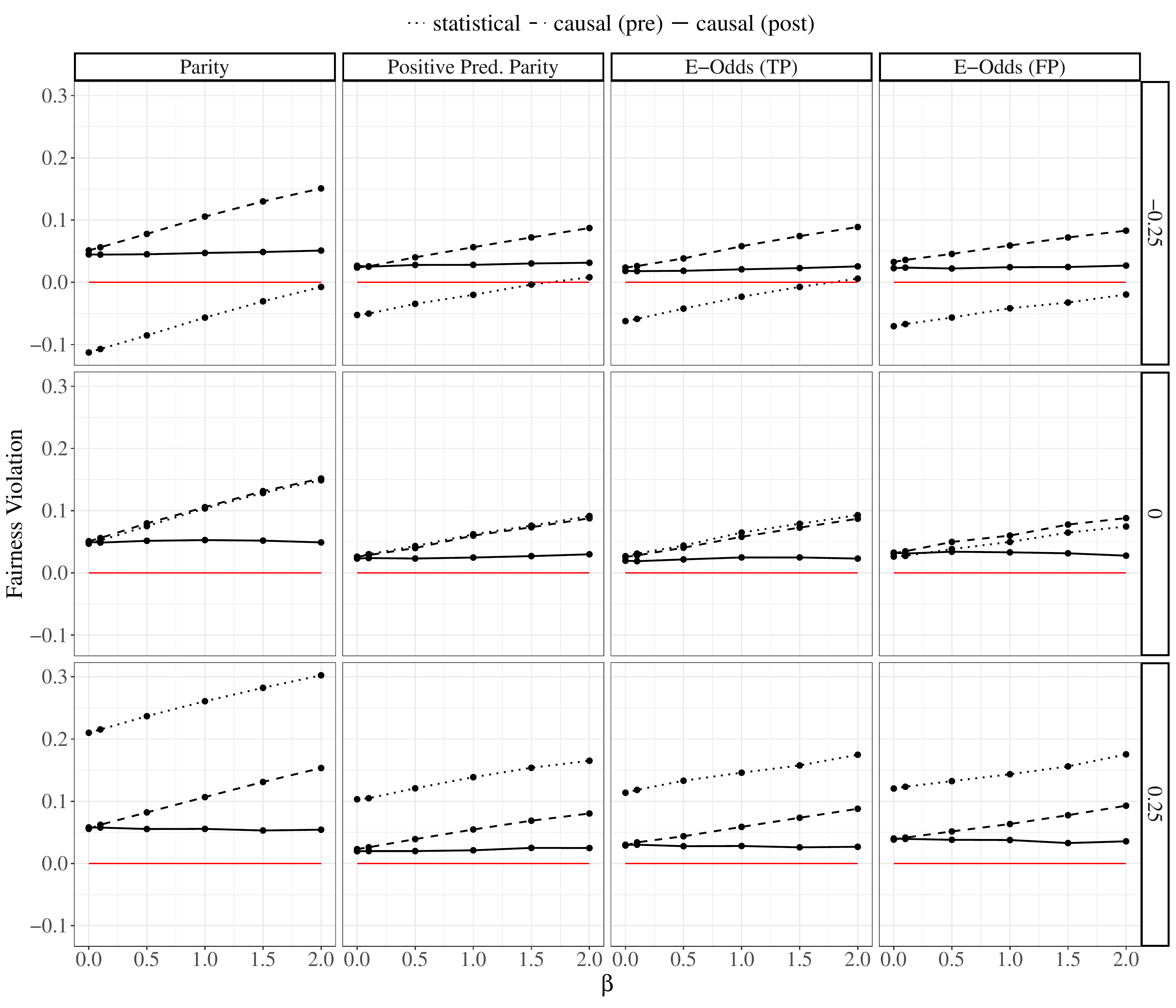}
    \caption{Stylized hiring scenario: line types indicate the unfairness evaluation approach. Causal (pre) and causal (post) measure pre- and post-interview causal disparity, respectively, and are compared against the statistical disparity line. Each panel corresponds to a different $\alpha$ value and fairness criteria. The $y-$axis exhibits violation of parity fairness criteria across values of $\beta$ and $\gamma = 0.2$.}
    \label{fig:evaluation-synthetic-big}
\end{figure}

\begin{figure}
    \centering
    \includegraphics[width = 0.9\textwidth]{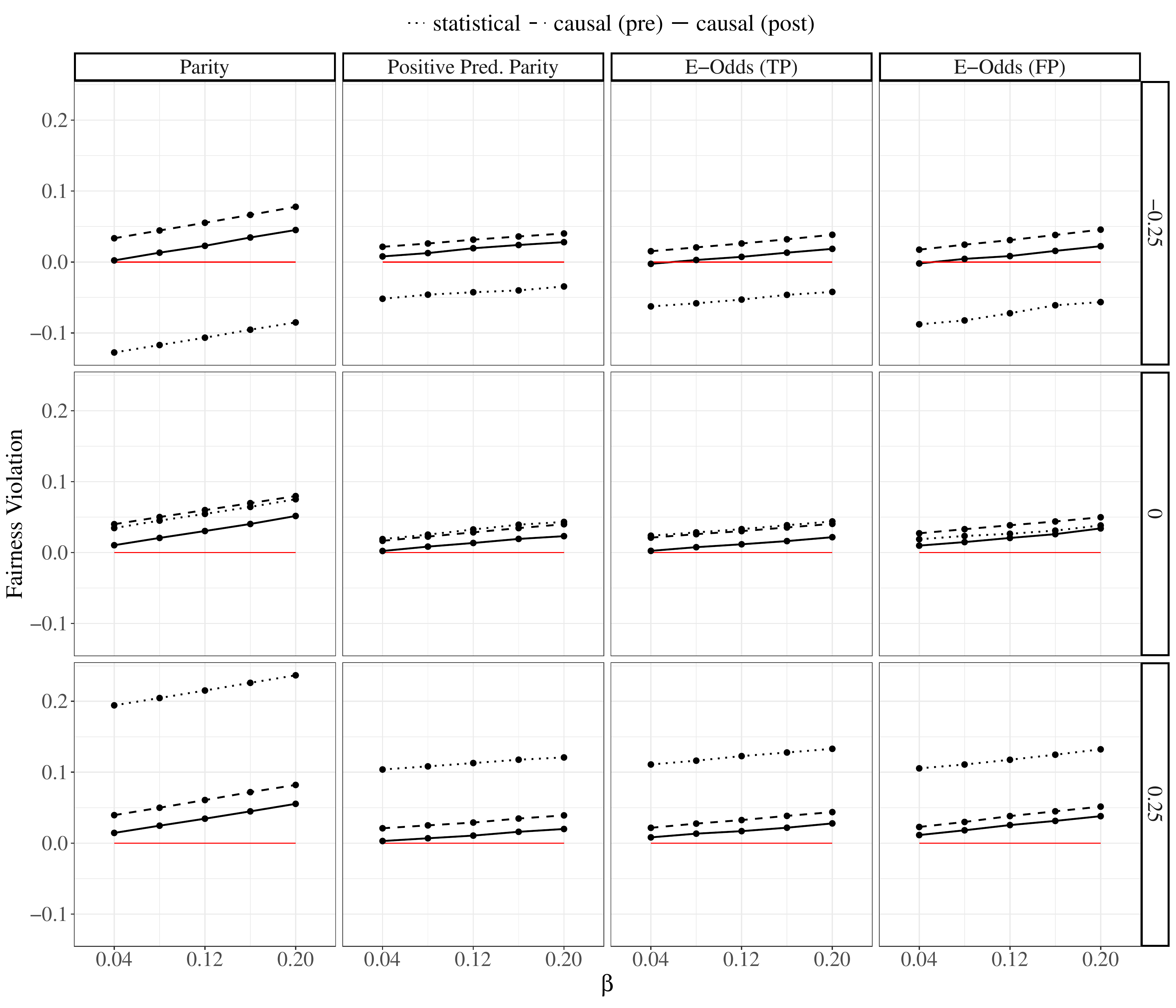}
    \caption{Stylized hiring scenario: line types indicate the unfairness evaluation approach. Causal (pre) and causal (post) measure pre- and post-interview causal disparity, respectively, and are compared against the statistical disparity line. Each panel corresponds to a different $\alpha$ value and fairness criteria. The $y-$axis exhibits violation of parity fairness criteria across values of $\gamma$ and $\beta = 0.5$.}
    \label{fig:evaluation-synthetic-big-gamma}
\end{figure}

\noindent\textbf{Police Stop Data Pre-processing. } We choose a slice of data after 2017, in New Orleans as it had the fewest columns with more than 70\% of the values missing. We further limit our analysis to vehicle stops (the data also contains person stops) and drivers whose race is either {Black} or {White}. We modify the decision on a search to include both frisk and vehicle search which are recorded in separate columns. That is because the order in which either one of the actions has been performed is not available, therefore, it is not immediately clear how one can treat them as separate decisions. We are primarily interested in evaluating the causal effect of race on the arrest decision. It is natural to choose the search stage as the time of intervention. In this case, decisions such as being searched, search outcome and being arrested become post-treatment variables. This is not the case if we choose the arrest stage as the point of intervention. We present results for both stages for comparison.

Determining the nature of the (hypothetical) intervention, however, is more challenging. For example, if the driver's license contained the race of the driver, we could possibly assume that the measured causal effect is associated with an intervention that modifies the reported race on the driver's license. In our analyses, we are making the assumption that the officer's perception of an individual's race maps to the binary White vs. Black records in the data. This means that the perception is not wrongfully influenced by factors such as the car one drives, the neighborhood or the time of day. Indeed, we are controlling for these variables by conditioning on them. Thus, the causal effects that we uncover can not be attributed to the influence of these features. We also assume that there are no multiple versions of Black or White which means that the level of the darkness of the skin color does not influence how an officer treats an individual. With respect to the positivity assumption, our data consists of both categorical and numerical variables. Ensuring that the positivity assumption holds is more challenging for numerical variables as it is very likely that there are not enough observations of both groups at each value of the numerical variables. However, it is possible to generalize inferences from neighboring values. In this work, we focus on positivity with respect to the categorical variables that have less generalizability power. We considered all intersections of the set of categorical variables in the data and removed those instances with only White or Black observations. As a result, 21\% of the data was removed. Finally, as confounders, we control for all the pre-treatment variables in the data including district, the hour of the day, the month of the year, officer assignment and the police zone, among others. We have also limited the location and time range of the data collection to control for confounding factors that affect policing practices over the years and locations (such as state laws).

Table~\ref{tab:mitigation-causal} shows the fairness violation, under causal metrics, across different methods. Here, negative numbers suggest the decision-making process is in favor of White individuals (being arrested is a negative outcome compared to being hired so lower values are preferred).  Similar to the hiring scenario, we observe that causal (post) reduces causal unfairness, however, it fails to eliminate the effect of the prior stage (search). In addition, statistical baselines exhibit different outcomes for White and Black individuals across all causal metrics. We hope that the results from the above experiments encourage adopting a causal perspective to unfairness evaluation and mitigation, especially as it provides a more interpretable framework to reason about the disparities. Indeed, as we showed, improving the causal criteria also leads to better observed outcomes across groups. Identifying situations where causal and statistical criteria are aligned is an interesting direction for future work.


\begin{table*}[t!]
    \small
    \centering
    \begin{tabular}{P{2.5cm}P{1.5cm}P{1.5cm}P{3cm}P{2cm}P{2cm}}
    \toprule[1pt]\midrule[0.3pt]
    \multirow{1}{*}{} & \multirow{2}{*}{Domain} & \multicolumn{4}{c}{Fairness Violation (Causal)}  \\
    \cmidrule{3-6} &  & Parity & Positive Pred. Parity & E-Odds (TP) & E-Odds (FP) \\
    \hline
    ReW  & \multirow{5}{*}{Hiring (pre)} & 0.079 & -0.049 & 0.071 & 0.087\\
    PRem & & 0.090 & -0.044 & 0.112 & 0.068 \\
    ROC & & 0.143 & -0.077 & 0.154 & 0.131 \\
    Causal (post) & & 0.084 & -0.021 & 0.111 & 0.040 \\
    Causal (pre) & & 0.000 & 0.000 & 0.000 & 0.000 \\
    \hline
    ReW  & \multirow{5}{*}{Policing (pre)} & -0.044 & 0.029 & 0.052 & -0.016 \\
    PRem &  & -0.040 & 0.032 & 0.063 & -0.015 \\
    ROC &  & -0.045 & 0.038 & 0.048 & -0.018 \\
    Causal (arrest) &  & -0.037 & 0.003 & -0.047 & -0.036 \\
    Causal (search) &  & 0.000 & 0.000 & 0.000 & 0.000 \\
    \midrule[0.3pt]\bottomrule
    \end{tabular}
\caption{Fairness violation of causal criteria for two decision-making systems: a synthetic hiring scenario and the police stop and search decision. Evaluations are conducted assuming that one intervenes at the beginning of the decision-making process. }
\label{tab:mitigation-causal}
\end{table*}

\end{document}